\documentclass[11pt,oneside,a4paper]{article}
\usepackage{fullpage}
\usepackage[affil-it]{authblk}
\usepackage{amsmath}
\usepackage{amsthm}
\usepackage{amsfonts}
\usepackage{amssymb}
\usepackage{mathtools}
\usepackage{multirow}
\usepackage[ruled,vlined]{algorithm2e}
\usepackage{algorithmic}
\usepackage{graphicx}
\usepackage{url}
\usepackage{diagbox}
\usepackage{dcolumn}

\newcolumntype{d}[1]{D{.}{.}{#1} }

\newtheorem{theorem}{Theorem}
\newtheorem{lemma}{Lemma}

\newtheorem{definition}{Definition}

\newcommand{\eg}{e.\,g., }
\newcommand{\ie}{i.\,e., }


\title{\textbf{k-hop graph neural networks}}
\author{Giannis Nikolentzos\textsuperscript{\rm 1,2}, George Dasoulas\textsuperscript{\rm 1,3} and Michalis Vazirgiannis\textsuperscript{\rm 1,2}\\ 
\textnormal{\textsuperscript{\rm 1}\'Ecole Polytechnique, France}\\
\textnormal{\textsuperscript{\rm 2}Athens University of Economics and Business, Greece}\\
\textnormal{\textsuperscript{\rm 3}Noah's Ark Lab, Huawei, France}
}
\date{}

\begin{document}
\maketitle 

\begin{abstract}
Graph neural networks (GNNs) have emerged recently as a powerful architecture for learning node and graph representations. 
Standard GNNs have the same expressive power as the Weisfeiler--Lehman test of graph isomorphism in terms of distinguishing non-isomorphic graphs.
However, it was recently shown that this test cannot identify fundamental graph properties such as connectivity and triangle freeness.
We show that GNNs also suffer from the same limitation.
To address this limitation, we propose a more expressive architecture, $k$-hop GNNs, which updates a node's representation by aggregating information not only from its direct neighbors, but from its $k$-hop neighborhood.
We show that the proposed architecture can identify fundamental graph properties.
We evaluate the proposed architecture on standard node classification and graph classification datasets.
Our experimental evaluation confirms our theoretical findings since the proposed model achieves performance better or comparable to standard GNNs and to state-of-the-art algorithms.
\end{abstract}

\section{Introduction}

In the past years, the amount of graph-structured data has grown steadily in a wide range of domains, such as in social networks and in chemoinformatics.
Learning useful representations from graph data is essential for many real-world applications.
For instance, in social network analysis, one might be interested in predicting the interests of users represented by the nodes of a network~\cite{yang2011like}.
In biology, an issue of high interest is the prediction of the functions of proteins modeled as graphs~\cite{borgwardt2005protein}.
These applications typically involve either node-focused or graph-focused tasks, with the difference lying in the entity of interest, a node or a graph.

Graph Neural Networks (GNNs) have recently emerged as a general framework addressing both node-related and graph-related tasks~\cite{bronstein2017geometric}.
Although numerous GNN variants have been proposed in the past
years~\cite{gori2005new,scarselli2009graph,li2015gated,defferrard2016convolutional,kearnes2016molecular,lei2017deriving,hamilton2017inductive,kipf2017semi,velickovic2018graph,zhang2018end}, they all share the same basic
idea, and can be reformulated into a single common framework.
Specifically, GNNs employ a message passing procedure, where each node updates its feature vector by aggregating the feature vectors of its neighbors~\cite{gilmer2017neural}.
After $k$ iterations of the message passing procedure, each node obtains a feature vector which captures the structural information within its $k$-hop neighborhood.
These representations can be used as features for node-related tasks.
For graph-related tasks, GNNs compute a feature vector for the entire graph using some permutation invariant readout function such as summing the feature vectors of all the nodes of the graph.

The GNN architectures have achieved state-of-the-art performance in many tasks such as in node classification and in link prediction.
Furthermore, they have seen considerable success in graph-related tasks even though they have faced intense competition from graph kernels, the approach that dominated the field for more than a
decade~\cite{nikolentzos2019graph}.
With the exception of the work of Scarselli~\cite{scarselli2008computational}, until recently, there has been little attempt to understand the properties and limitations of GNNs.
It was clear though that there is a close connection between GNNs and the Weisfeiler--Lehman (WL) test of graph isomorphism~\cite{weisfeiler1968reduction}, a powerful heuristic which can successfully test isomorphism for a
broad class of graphs~\cite{babai1979canonical}.
Similar to GNNs, the WL test iteratively updates a given node's label by aggregating the labels of its neighbors.
Specifically, the algorithm augments the label of each node by the sorted multiset of labels of neighboring nodes, and compresses these augmented labels into new, short labels.
To test graph isomorphism, this relabeling procedure is repeated until the label sets of the input graphs are not identical, or the number of iterations reaches a specific value.

Some recent studies have made attempts to formally characterize the expressive power of GNNs~\cite{morris2018weisfeiler,xu2018powerful}.
These studies have compared the expressiveness of GNNs with that of the WL test, and have shown that GNNs do not have more power in terms of distinguishing between non-isomorphic graphs than the WL algorithm.
To make matters worse, it was recently shown that the WL subtree kernel (which capitalizes on the WL test) has insufficient expressive power for identifying fundamental graph properties~\cite{kriege2018property}.
It remains though unclear how GNNs encode subgraph/graph information into their learned representations, and whether they can identify such properties.
Since these architectures are directly related to the WL test, they may also lack expressive power, and they may fail to identify these properties.
Hence, the need for more powerful representations is more essential than ever.

\textbf{Present Work.}
In this paper, we further analyze the representational power of GNNs.
Specifically, we study if GNNs can identify specific properties of graphs.
We say that a GNN identifies a property if no two graphs are mapped to the same feature vector unless they both have or both do not have the property. 
We demonstrate that the standard GNN fails to identify fundamental graph properties such as connectivity, bipartiteness and triangle-freeness.
We show that this limitation of GNNs stems from the myopic nature of the message-passing procedure which only considers the direct neighbors of each node.
To account for that, we propose a novel architecture, called $k$-hop-GNNs, which takes into account not only the immediate neighbors of each node, but its whole $k$-hop neighborhood.
By updating node features using not only the direct neighbors, but taking into account the entire $k$-hop neighborhood, we can capture structural information that is not visible when aggregating only the $1$-hop neighborhood.
The proposed model is strictly more powerful than the standard GNN architecture.
Furthermore, in contrast to the GNN framework, the proposed architecture is capable of distinguishing global properties such as connectivity.
We demonstrate the proposed architecture in a variety of node and graph classification tasks.
The results show that the proposed $k$-hop-GNNs are able to consistently outperform traditional GNNs on most datasets.
Our main contributions are summarized as follows:
\begin{itemize}
    \item We show that standard GNNs cannot identify essential graph properties such as connectivity, bipartiteness and triangle-freeness.
    \item We propose $k$-hop-GNNs, a novel architecture for performing machine learning on graphs which is more powerful than traditional GNNs.
    \item We evaluate the proposed architecture on several node classification and graph classification datasets, and achieve performance better or comparable to standard GNNs and to state-of-the-art algorithms.
\end{itemize}

The rest of this paper is organized as follows.
Section~\ref{sec:preliminaries} introduces some preliminary concepts and summarizes the standard graph neural network model.
Section~\ref{sec:limitations} analyzes the expressive power of the graph neural network model highlighting its limitations.
Section~\ref{sec:contribution} presents the proposed model for performing machine learning tasks on graph-structured data, and shows that it is theoretically more powerful than the standard graph neural network architecture.
Section~\ref{sec:experiments} evaluates the proposed architecture on several standard datasets.
Finally, section~\ref{sec:conclusion} concludes.

\section{Preliminaries}\label{sec:preliminaries}
We start by fixing our notation, and then we present the standard graph neural network framework.

\subsection{Notation}

Let $G = (V,E)$ be an undirected graph consisting of a set $V$ of nodes and a set $E$ of edges between them.
We will denote by $n$ the number of nodes and by $m$ the number of edges.
The neighborhood of radius $k$ (or $k$-hop neighborhood) of a node $v \in V$ is the set of nodes at a distance less than or equal to $k$ from $v$ and is denoted by $\mathcal{N}_k(v)$.
Given a set of nodes $S \subseteq V$, the subgraph induced by $S$ is a graph that has $S$ as its node set and it contains every edge of $G$ whose endpoints are in $S$.
The neighborhood subgraph of radius $k$ of a node $v \in V$ is the subgraph induced by the neighborhood of radius $k$ of $v$ and $v$ itself, and is denoted by $G_v^k$.
A node-labeled graph is a graph with labels on nodes.
That is, given a set of labels $\mathcal{L}$, there exists a function $\ell : V \rightarrow \mathcal{L}$ that assigns labels to the nodes of the graph.
Instead of discrete labels, nodes may also be annotated with multiple categorical or real-valued properties.
These graphs are known as attributed graphs.

\subsection{Graph Neural Networks}

Let $G=(V,E)$ be a graph.
Suppose each vertex $v \in V$ is annotated with a feature vector $h_v^{(0)} \in \mathbb{R}^d$.
For graphs with discrete node labels, these vectors usually correspond to some representations (\eg one-hot encoding) such that two nodes $v$ and $u$ have identical feature vectors if and only if $\ell(v) = \ell(u)$.
For attributed graphs, these feature vectors may be set equal to the attribute vectors of the nodes.
 For graphs without node labels and node attributes, these vectors can be initialized with a collection of local vertex features that are invariant to vertex renumbering (\eg degree, $k$-core number, number of triangles, etc.).
A GNN model consists of a series of neighborhood aggregation layers.
Each one of these layers uses the graph structure and the node feature vectors from the previous layer to generate new representations for the nodes.
The feature vectors are updated by aggregating local neighborhood information.

Suppose we have a GNN model that contains $T$ neighborhood aggregation layers.
In the $t$-th neighborhood aggregation layer ($t > 0$), the hidden state $h_v^{(t)}$ of a node $v$ is updated as follows:
\begin{equation}
    \begin{split}
        a_v^{(t)} &= \text{AGGREGATE}^{(t)}  \Bigl( \Bigl\{ h_u^{(t-1)} \vert  u \in \mathcal{N}_1(v) \Bigr\} \Bigr) \\
        h_v^{(t)} &= \text{MERGE}^{(t)}  \Bigl( h_v^{(t-1)}, a_v^{(t)}  \Bigr)
    \end{split}
    \label{eq:gnn_general}
\end{equation}
By defining different $\text{AGGREGATE}^{(t)}$ and $\text{MERGE}^{(t)}$ functions, we obtain a different GNN variant.
For the GNN to be end-to-end trainable, both functions need to be differentiable.
Furthermore, since there is no natural ordering of the neighbors of a node, the $\text{AGGREGATE}^{(t)}$ function must be permutation invariant.
There are numerous concrete implementations of the above GNN framework.
Some of them integrate the $\text{AGGREGATE}^{(t)}$ and $\text{MERGE}^{(t)}$ steps into a single function~\cite{kipf2017semi,zhang2018end} as follows:
\begin{equation*}
    h_v^{(t)} = \frac{1}{|\mathcal{N}_1(v)| +1} \sum_{u \in \mathcal{N}_1(v) \cup \{ v \}} \text{MLP}^{(t)} \bigl( h_u^{(t-1)} \bigr) 
\end{equation*}
where $\text{MLP}^{(t)}$ is a multi-layer perceptron of the $t$-th neighborhood aggregation layer.
Note that the majority of the proposed models use $1$-layer perceptrons instead of MLPs.
Another widely-used GNN model is implemented as follows~\cite{lei2017deriving}:
\begin{equation*}
    \begin{split}
        a_v^{(t)} &= \frac{1}{|\mathcal{N}_1(v)|} \sum_{u \in \mathcal{N}_1(v)} \text{MLP}_1^{(t)} \bigl( h_u^{(t-1)} \bigr) \\
        h_v^{(t)} &= \text{MLP}_2^{(t)} \bigl( h_v^{(t-1)} + a_v^{(t)} \bigr)
    \end{split}
\end{equation*}
where again $\text{MLP}_1^{(t)}$ and $\text{MLP}_2^{(t)}$ are multi-layer perceptrons of the $t$-th neighborhood aggregation layer.

For node-level tasks, the node feature vectors $h_v^{(T)}$ of the final neighborhood aggregation layer are usually passed on to a fully-connected neural network.
For graph-level tasks, GNNs apply a $\text{READOUT}$ function to node representations generated by the final neighborhood aggregation layer to obtain a vector representation over the whole graph:
\begin{equation}
    h_G = \text{READOUT} \Bigl( \Bigl\{ h_v^{(T)} \vert  v \in V \Bigr\} \Bigr)
    \label{eq:readout}
\end{equation}
Similarly to the $\text{AGGREGATE}^{(t)}$ function, the $\text{READOUT}$ function is necessary to be differentiable and permutation invariant.
A common $\text{READOUT}$ function computes the mean of the representations of the nodes:
\begin{equation*}
    h_G = \frac{1}{|V|} \sum_{v \in V}  h_v^{(T)}
\end{equation*}
However, there have also been proposed more sophisticated functions based on sorting~\cite{zhang2018end}, on concatenation across the iterations/layers~\cite{xu2018powerful} and on
clustering~\cite{such2017robust,ying2018hierarchical}.

\section{Limitations of the Standard GNN Model}\label{sec:limitations}
To gain theoretical understanding of the properties and weaknesses of GNNs, we capitalize on concepts introduced by Goldreich in the context of property testing~\cite{goldreich2017introduction}, and further refined by Kriege \textit{et al.} for investigating the expressive power of graph kernels~\cite{kriege2018property}.

Let $\mathcal{G}_n$ be the set of graphs on $n$ vertices, where $n \in \mathbb{N}$.
A graph property is a set $\mathcal{P}$ of graphs that is closed under isomorphism.
We denote the set of graphs in $\mathcal{P}$ on $n$ vertices by $\mathcal{P}_n$.
In this paper we study the following three fundamental graph properties: ($1$) connectivity, ($2$) bipartiteness, and ($3$) triangle-freeness. 
A graph is \textit{connected} if there is a path from any vertex to any other vertex in the graph.
A graph $G=(V,E)$ is \textit{bipartite} if its set of vertices $V$ can be decomposed into two disjoint sets $V_1$ and $V_2$, \ie $V=V_1 \cup V_2$, such that every edge $e \in E$ connects a vertex in $V_1$ to a vertex in $V_2$.
Finally, a graph is \textit{triangle-free} if it does not contain a triangle (\ie a cycle of three vertices).
Following~\cite{kriege2018property}, we say that a GNN can identify a property if no two graphs obtain the same representation unless they both have or both do not have the property.
\begin{definition}
    Let $\mathcal{P}$ be a graph property.
    If for each $n \in \mathbb{N}$, a GNN produces different representations for every $G_1 \in \mathcal{P}_n$ and $G_2 \not \in \mathcal{P}_n$, \ie it holds that $h_{G_1} \neq h_{G_2}$, then we say that $\mathcal{P}$ can be identified by the GNN.
\end{definition}
We next study if the standard GNN architecture can identify the above three graph properties.
We assume that either all nodes or nodes with the same degree are annotated with the same feature vector.
We first show that the standard GNN produces exactly the same representation for the nodes of all regular graphs of a specific degree in $\mathcal{G}_n$ for some $n \in \mathbb{N}$.
\begin{lemma}
    The standard GNN maps the nodes of two regular graphs of the same size and degree to the same feature vector.
    \label{lem:same_mapping}
\end{lemma}
\begin{proof}
    Let $G_1=(V_1, E_1)$ and $G_2=(V_2, E_2)$ be two non-isomorphic regular graphs of the same degree with the same number of vertices.
    We show for an arbitrary iteration $t \geq 1$ and nodes $v_1 \in V_1$, $v_2 \in V_2$ that $h_{v_1}^{(t)} = h_{v_2}^{(t)}$.
    All nodes have the same initial representation, hence, in iteration $0$, it holds that $h_{v_1}^{(0)} = h_{v_2}^{(0)}$.
    Assume for induction that $h_{v_1}^{(t-1)} = h_{v_2}^{(t-1)}$.
    Let $\mathcal{M}_{v_1} = \{ h_{u_1}^{(t-1)} | u_1 \in \mathcal{N}_1(v_1) \}$ and $\mathcal{M}_{v_2} = \{ h_{u_2}^{(t-1)} | u_2 \in \mathcal{N}_1(v_2) \}$ be the multisets of feature vectors of the neighbors of $v_1$ and $v_2$, respectively.
    By the induction hypothesis, we know that $\mathcal{M}_{v_1} = \mathcal{M}_{v_2}$ and that $h_{v_1}^{(t-1)} = h_{v_2}^{(t-1)}$ such that independent of the choice of the $\text{AGGREGATE}^{(t)}$ and $\text{MERGE}^{(t)}$ functions in~\eqref{eq:gnn_general}, we get $h_{v_1}^{(t)} = h_{v_2}^{(t)}$.
    This holds as the input to both functions $\text{AGGREGATE}^{(t)}$ and $\text{MERGE}^{(t)}$ is identical.
    This proves that $h_{v_1}^{(t)} = h_{v_2}^{(t)}$, and thereby the lemma.
\end{proof}

\begin{figure}[t]
    \centering
    \includegraphics[width=.6\textwidth]{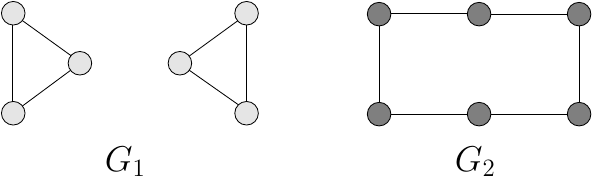}
    \caption{Two $2$-regular graphs on $6$ vertices. The two graphs serve as a counterexample for the proof of Theorem~\ref{theorem:limit}.}
    \label{fig:example1}
\end{figure}

Note that the Lemma implies that two regular graphs with the same node degree and the same size also have the same graph representation because the READOUT function receives the same input.
We next show that the GNN architecture cannot identify the three graph properties defined above since for each one of these properties, there exists one regular graph in $\mathcal{P}_n$ and another regular graph of the same degree in $\mathcal{G}_n \setminus \mathcal{P}_n$.
\begin{theorem}
    The standard GNN cannot identify connectivity, bipartiteness or triangle freeness.
    \label{theorem:limit}
\end{theorem}
\begin{proof}
    Consider a cycle with six vertices (graph $G_2$) and two triangles with three vertices (graph $G_1$) as illustrated in Figure~\ref{fig:example1}.
    Both $G_1$ and $G_2$ are regular graphs of the same degree with the same number of vertices.
    Hence, according to Lemma~\ref{lem:same_mapping}, after $T$ neighborhood aggregation steps, the nodes of both graphs have obtained identical representations, \ie $h_{v}^{(T)} = h_{u}^{(T)}, \, \forall v,u \in V_1 \cup V_2$.
    Therefore, independent of the choice of the $\text{READOUT}$ function in~\eqref{eq:readout}, the two graphs will have identical representations, $h_{G_1} = h_{G_2}$, since the input to the $\text{READOUT}$ function is identical.
    Clearly, $G_1$ is disconnected, while $G_2$ is connected.
    Hence, these two graphs correspond to a counterexample to the distinguishability of connectivity.
    Furthermore, consider the graphs $G_3$ and $G_4$ as illustrated in Figure~\ref{fig:example2}.
    Note that $G_3$ contains triangles, but is not bipartite, whereas $G_4$ is bipartite, and triangle-free.
    Both $G_3$ and $G_4$ are regular graphs of the same degree with the same number of vertices.
    Therefore, they obtain identical representations $h_{G_3} = h_{G_4}$, and correspond thus to a counterexample to the distinguishability of the above two properties.
\end{proof}

\begin{figure}[t]
    \centering
    \includegraphics[width=.6\textwidth]{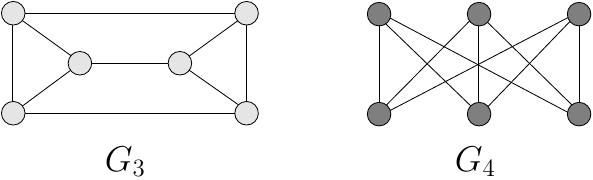}
    \caption{Two $3$-regular graphs on $6$ vertices. The two graphs serve as a counterexample for the proof of Theorem~\ref{theorem:limit}.}
    \label{fig:example2}
\end{figure}

We should note here that there exist complex network measures which have served as features for node classification tasks in previous studies and which can capture some of these graph
properties~\cite{silva2012network,carneiro2017organizational,cupertino2018scheme}.

\section{k-hop Graph Neural Networks}\label{sec:contribution}
In this section, we propose a generalization of GNNs, so-called $k$-hop Graph Neural Networks ($k$-hop GNNs).
This new model consists of neighborhood aggregation layers that do not take into account only the direct neighbors of the nodes, but their entire $k$-hop neighborhood.
Hence, instead of the neighborhood aggregation layer shown in~\eqref{eq:gnn_general}, the proposed model updates the hidden state $h_v^{(t)}$ of a node $v$ as follows:
\begin{equation}
    \begin{split}
        a_v^{(t)} &= \text{AGGREGATE}^{(t)}  \Bigl( \Bigl\{ h_u^{(t-1)} \vert  u \in \mathcal{N}_k(v) \Bigr\} \Bigr) \\
        h_v^{(t)} &= \text{MERGE}^{(t)}  \Bigl( h_v^{(t-1)}, a_v^{(t)}  \Bigr)
    \end{split}
    \label{eq:gnn_proposed}
\end{equation}
We next present an instance of the proposed architecture which is strictly stronger than standard GNNs in terms of distinguishing non-isomorphic graphs, and is capable of identifying graph properties which are not captured by the standard GNN architecture.

\subsection{Proposed Architecture}

Let $G = (V,E)$ be a graph.
In what follows, we will focus on a single node $v \in V$, and we will present how the representation of this node is updated during the neighborhood aggregation phase.
Node $v$ will also be referred as the root of the $k$-hop neighborhood subgraph $G_v^k$.
For a given iteration/layer $t$, and a root node $v$, we define an \textit{inner} representation $x_u$ of each node $u \in \mathcal{N}_k(v)$ and we initialize it as $x_u = h_u^{(t-1)}$.
We will next describe how the hidden state $h_v^{(t)}$ of the root $v$ is computed.
Let $\text{UPDATE}(w, S)$ denote a module which takes as input a node $w$ and a set of nodes $S$, and is defined as follows:
\[
    \text{UPDATE}(w, S) = \text{MLP} \Bigl( \text{MLP}_1 \bigl(x_w\bigr) + \sum_{u \in S} \text{MLP}_2 \bigl( x_u \bigr) \Bigr)
\]
where $\text{MLP}, \text{MLP}_1, \text{MLP}_2$ are multi-layer perceptrons and $x_w, x_u$ are the \textit{inner} representations of nodes $w$ and $u$, respectively.
The proposed approach uses a series of $\text{UPDATE}$ modules to update the representations of the nodes that belong to the $k$-hop neighborhood of $v$, following a sequential procedure from the most distant ones to the direct neighbors of $v$.
Although the neural network learns a new vector representation for some of the nodes $u \in \mathcal{N}_k(v)$, these feature vectors are only calculated in the context of updating the root node's representation.
Hence, after computing the new representation $h_v^{(t)}$ of $v$, these representations are not useful any more.
These \textit{inner} representations should not be confused with the $h_u^{(t)}$ representation that the network learns for each of these nodes by taking into account their own $k$-hop neighborhoods.

Let $R_d(v)$ denote the set of nodes at distance (hop count) exactly $d > 0$ from $v$.
Hence, $R_1(v) = \mathcal{N}_1(v)$ is the set of direct neighbors of $v$, while $R_d(v)$ denotes the ring of nodes at distance $d$, which we refer to as the nodes at level $d$.
Note that the neighbors of a node $u \in R_d(v)$ belong to one of the next three sets: $R_{d-1}(v)$, $R_d(v)$ or $R_{d+1}(v)$.
Specifically, $u$ cannot be connected with nodes at levels $l > d+1$ because then these nodes would belong to level $d+1$ instead of $l$.
Furthermore, $u$ cannot be connected with nodes at levels $l < d-1$ because then $u$ would belong to some level smaller than $d$.
Given a node $u \in \mathcal{N}_k(v)$, the \textit{inner} representation $x_u$ of $u$ is updated at most twice.
The two updates aggregate information from the neighbors of $u$ that are located at the immediately higher and at the same level of the neighborhood subgraph, respectively.
Hence, the proposed model performs the following two types of updates of \textit{inner} representations: ($1$) updates across rings of nodes, and ($2$) updates within a ring of nodes.
We next present these two updates in detail:
\begin{itemize}
    \item Updates \textit{\textbf{across}} rings of nodes:
    Let $u \in R_d(v)$ be a node that belongs to the $k$-hop neighborhood of $v$ and whose shortest path distance from $v$ is equal to $d$.
    Let also $\mathcal{B} = \mathcal{N}_1(u) \cap R_{d+1}(v)$ denote the neighbors of $u$ that belong to level $d+1$ of $G_v^k$.
    Note that $\mathcal{B}$ is empty if $k=d$ or if all the neighbors of $u$ belong to levels $d-1$ and $d$ of $G_v^k$.
    If $\mathcal{B}$ is not empty, the representation $x_u$ of $u$ is updated as follows:
    \[
        x_u = \text{UPDATE}_{d, across}^{(t)} \bigl( u, \mathcal{B} \bigr)
\]
    Otherwise, if $u$ has no neighbors at the next higher level (\ie $\mathcal{N}_1(u) \cup R_{d+1}(v)$ is empty), then its representation is not updated.
    
    \item Updates \textbf{\textit{within}} a ring of nodes:
    If $u$ has one or more neighbors at the same level of $G_v^k$, and hence $\mathcal{D} = \mathcal{N}_1(u) \cap R_{d}(v)$ is not empty, the representation $x_u$ of $u$ is re-updated as follows:
    \[
        x_u = \text{UPDATE}_{d, within}^{(t)} \bigl( u, \mathcal{D} \bigr)
\]
    Otherwise, if $\mathcal{D} = \emptyset$, its representation is not updated.
\end{itemize}

The proposed approach starts from the most distant nodes and follows a sequential procedure updating the feature vectors of nodes that are gradually closer to the root.
The first type of update (across rings) precedes the second (within a ring).
After all its direct neighbors $u \in \mathcal{N}_1(v)$ have been processed, the hidden state of the root node $v$ is computed as follows:
\[
    h_v^{(t)} = \text{UPDATE}_{0, across}^{(t)} \bigl( v, \mathcal{N}_1(v) \bigr)
\]
As mentioned above, although the model learns a new \textit{inner} representation $x_u$ for some of the nodes in $\mathcal{N}_k(v)$, these representations are only learned for the purpose of updating the root node's hidden state.
Furthermore, it is clear that for a single neighborhood aggregation layer, the proposed model needs at most $2k$ \text{UPDATE} modules.
As we will show next, the proposed model can capture the structural information within the root node's $k$-hop neighborhood even if it comprises of a single neighborhood aggregation layer.
Hence, instead of using multiple neighborhood aggregation layers/iterations, it is more suitable to increase the value of $k$.
The various steps of the proposed model are illustrated in Algorithm~\ref{alg:k_hop_gnn}.
We provide below a simple example that illustrates the update procedure that was presented above.

\begin{algorithm}[t]
    \caption{$k$-hop GNN}
    \label{alg:k_hop_gnn}
    \textbf{Input}: Graph $G = (V,E)$, node features $\{h_v : v \in V\}$, number of neighborhood aggregation layers $T$, number of hops $k$\\
    \textbf{Output}: Node features $\{h_v^{(T)} : v \in V\}$\\
    \begin{algorithmic}[1] 
    \FOR{ $t \in \{1,\ldots,T\}$}
        \FOR{$v \in V$}
            \FOR{$u \in R_k(v)$}
                \STATE $\mathcal{D} \longleftarrow \mathcal{N}_1(u) \cap R_{k}(v)$\\
                \STATE $x_u \longleftarrow \text{UPDATE}_{k, within}^{(t)} \big( u, \mathcal{D} \big)$
            \ENDFOR
            \FOR{$i \in \{k-1,\ldots,1\}$}
                \FOR{$u \in R_i(v)$}
                    \STATE $\mathcal{B} \longleftarrow \mathcal{N}_1(u) \cap R_{i+1}(v)$\\
                    \STATE $x_u \longleftarrow \text{UPDATE}_{i, across}^{(t)} \big( u, \mathcal{B} \big)$\\
                    \STATE $\mathcal{D} \longleftarrow \mathcal{N}_1(u) \cap R_{i}(v)$\\
                    \STATE $x_u \longleftarrow \text{UPDATE}_{i, within}^{(t)} \big( u, \mathcal{D} \big)$
                \ENDFOR
            \ENDFOR
            \STATE $h_v^{(t)} = \text{UPDATE}_{0, across}^{(t)} \big( v, \mathcal{N}_1(v) \big)$
        \ENDFOR
    \ENDFOR
    \end{algorithmic}
\end{algorithm}

After $T$ iterations (\ie $T$ neighborhood aggregation layers), the emerging node feature vectors $h_v^{(T)}$ can be used in any node-related task.
For graph-level tasks, the proposed model can compute a vector representation over the whole graph by applying a $\text{READOUT}$ function similar to the one shown in~\eqref{eq:readout}.

\subsection{Example}\label{sec:example}
We next provide a simple example that illustrates the update procedure that was presented above.
Specifically, Figure~\ref{fig:example3} shows the $2$-hop neighborhood graph $G_{v_1}^2$ of a node $v_1 \in V$.
As mentioned above, we first consider the most distant nodes (\ie nodes $v_4$ and $v_5$ since $v_4, v_5 \in R_2(v_1)$).
The representations of nodes $v_4$ and $v_5$ are not updated since these two nodes are at the frontier of $G_{v_1}^2$.
Furthermore, there is no edge between the two nodes.
However, these two nodes contribute to the update of the representation of node $v_2$.
Specifically, the \textit{inner} representation of node $v_2$ is updated as follows:
\[
    x_{v_2} = \text{UPDATE}_{1, across}^{(t)} \bigl(v_2, \{v_4, v_5\} \bigr)
\]
Then, we update the \textit{inner} representations of nodes whose shortest path distance from the root is $1$ and which are connected with other nodes with the same distance from the root.
There is one such pair of nodes (\ie nodes $v_2$ and $v_3$) which are updated as:
\[
    \begin{split}
        x_{v_2} &= \text{UPDATE}_{1, within}^{(t)} \bigl(v_2, \{v_3\} \bigr)\\
        x_{v_3} &= \text{UPDATE}_{1, within}^{(t)} \bigl(v_3, \{v_2\} \bigr) 
    \end{split}
\]
Finally, we update the root by aggregating information from its direct neighbors:
\[
    h_{v_1}^{(t)} = \text{UPDATE}_{0, across}^{(t)} \bigl(v_1, \{v_2, v_3\} \bigr) 
\]

\begin{figure}
    \centering
    \includegraphics[width=.3\textwidth]{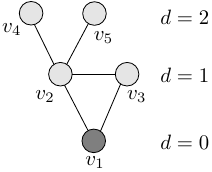}
    \caption{The $2$-hop neighborhood graph $G_{v_1}^2$ of a node $v_1$ of graph $G$.}
    \label{fig:example3}
\end{figure}

\subsection{Expressive Power}

We next study the identifiability of the proposed $k$-hop GNN.
The following Theorem comprises the main results about graph properties that can be identified by the proposed model.
\begin{theorem}
    For the $k$-hop GNN, there exists a sequence of modules $\text{UPDATE}_{0,across}^{(0)}$, $\text{UPDATE}_{1,within}^{(0)}$, $\text{UPDATE}_{1,across}^{(0)}, \ldots \dots, \text{UPDATE}_{k-1, across}^{(T)}$, $\text{UPDATE}_{k,within}^{(T)}$ such that
    \begin{enumerate}
    \setlength\itemsep{0cm}
        \item it can identify triangle-freeness for $k \geq 1$
        \item connectivity for $k > \delta_{min}$ where $\delta_{min}$ is the minimum of the diameters of the connected components
        \item bipartiteness for $k \geq \frac{l-1}{2}$ where $l$ is the length of the smallest odd cycle in the graph (if any)
    \end{enumerate}
    \label{thm:expressiveness}
\end{theorem}
\begin{proof}
    The proof is left to the Appendix.
\end{proof}
It should be mentioned that there have already been proposed GNNs that update the representations of the nodes based on their $k$-hop neighborhoods.
Such GNNs employ polynomials of order $k$~\cite{defferrard2016convolutional} or auto-regressive moving average (ARMA) filters~\cite{bianchi2019graph} to approximate a transfer function that acts on the
eigenvalues of the normalized Laplacian matrix, and can be shown that they utilize information from the nodes' $k$-hop neighborhoods.
However, in our experimental evaluation, these models fail to consistently capture the above three properties.
Furthermore, other methods take into account higher-order neighborhoods by computing powers of the adjacency matrix~\cite{abu2019mixhop}.

\subsection{Computational Complexity}\label{sec:complexity}
The increase of expressiveness provided by the $k$-hop GNN model does not come without a price.
Clearly the time complexity of the proposed model is higher than that of the standard GNN. The computational steps of the proposed method consist of two parts:
\begin{enumerate}
    \item \textbf{Preprocessing step:} In this phase, the model extracts the neighborhood subgraphs of all nodes.
    For each node, the neighborhood subgraph can be extracted in linear time in the number of edges of the neighborhood.
    Hence, the complexity of the preprocessing step is $\mathcal{O}(nm)$ in the worst case (\ie for a complete graph).
    For sparse graphs, it can become significantly lower, \ie $\mathcal{O}(n \, \bar{d}^k)$ where $\bar{d}$ is the average degree of the nodes.
    We should note that this step is computed only once.
    \item \textbf{Message passing procedure:} To compute the representation of each node, for each edge of its neighborhood subgraph, a message needs to be sent from some node to another node.
    Therefore, the complexity of one message passing iteration (\ie one epoch in our implementation) is $\mathcal{O}(nm)$ in the worst case and $\mathcal{O}(n \, \bar{d}^k)$ for sparse graphs.
    Note however, that the proposed message passing layer allows a GPU-friendly implementation.
    Therefore, in practice, as verified by our experiments, the running time is not prohibitive.
\end{enumerate}
To sum up, the total computational complexity of $k$-hop Graph Neural Network is $\mathcal{O}(nm)$, while for sparse graphs is $\mathcal{O}(n \, \bar{d}^k)$.
On the other hand, the complexity of the standard GNN model is $\mathcal{O}(km)$ where $k$ is the number of message passing iterations.
For sparse graphs, and when the number of hops $k$ is set to a small value, the complexity of the proposed model is comparable to the complexity of the standard GNNs.

\section{Experimental Evaluation}\label{sec:experiments}
In this Section, we evaluate the performance of the proposed $k$-hop GNN in two tasks: ($1$) node classification, ($2$) graph classification. 

\subsection{Node Classification}

The main objective of node classification is to assign class labels to unlabeled nodes.
We evaluate the proposed model on synthetic graphs with planted structural equivalences and on one real-world dataset.
Each node $v_i \in V$ has an associated class label $y_i$ and the goal is to learn a representation vector $h_{v_i}^{(T)}$ of $v_i$ such that $v_i$'s label can be predicted as $y_i = f(h_{v_i}^{(T)})$.

\subsubsection{Synthetic Datasets}\label{sec:experiments_node_synthetic}

\noindent\textbf{Datasets.}
To generate the graphs, we follow the same procedure as in~\cite{donnat2018learning}.
Structurally equivalent nodes are assigned the same class labels.

All the generated graphs consist of a cycle of length $40$ and some basic shapes (``house'', ``fan'', ``star'') which are regularly placed along the cycle.
In the ``basic'' setup, $10$ instances of only one of the three types (randomly chosen with uniform probability) are placed along the cycle.
In the ``varied'' setup, $10$ instances of each one of the three shapes are randomly placed along the cycle.
The use of multiple shapes increases the number and complexity of the structural role patterns, posing a challenge to the learning algorithms.
To assess how the algorithms perform in noisy scenarios, we introduce two additional configurations (``basic perturbed'' and ``varied perturbed'') where we add edges uniformly at random on the generated graphs.
The number of edges that are added is equal to $10\%$ of the edges of the graph.
The shapes that are placed along the cycle graph in the different setups are illustrated in Table~\ref{tab:shapes}.

\begin{table}[t]
    \centering
    \scriptsize
    \def\arraystretch{1.1}
    \begin{sc}
    \begin{tabular}{|l|c|} \hline
    Configuration & Shapes placed along a cycle graph \\ \hline
    \multirow{5}{*}{basic} & \multirow{5}{*}{\includegraphics[width=.4\textwidth]{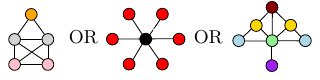}} \\ 
    & \\
    & \\
    & \\
    & \\ \hline
    \multirow{5}{*}{\parbox{1cm}{basic\\ perturbed}} & \multirow{5}{*}{\includegraphics[width=.4\textwidth]{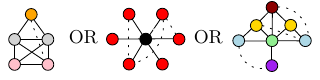}} \\
    & \\
    & \\
    & \\
    & \\ \hline
    \multirow{5}{*}{varied} & \multirow{5}{*}{\includegraphics[width=.4\textwidth]{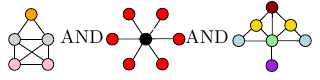}} \\ 
    & \\
    & \\
    & \\
    & \\ \hline
    \multirow{5}{*}{\parbox{1cm}{varied\\ perturbed}} & \multirow{5}{*}{\includegraphics[width=.4\textwidth]{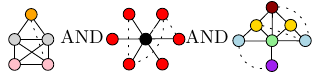}} \\ 
    & \\
    & \\
    & \\
    & \\ \hline
    \end{tabular}
    \end{sc}
    \caption{Example of synthetically generated structures for each configuration. The different colors denote structurally equivalent nodes. Dashed lines denote perturbed graphs (obtained by randomly adding edges).}
    \label{tab:shapes}
\end{table}

\noindent\textbf{Baselines.}
We compare the proposed model against $3$ recent state-of-the-art techniques for learning structural node representations: ($1$) RolX~\cite{henderson2012rolx}, ($2$)
struc2vec~\cite{ribeiro2017struc2vec}, and ($3$) GraphWave~\cite{donnat2018learning}.
Note that these $3$ algorithms are unsupervised, in the sense that they only take a graph as input and not any class labels of the nodes.
We also compare the $k$-hop GNN model against the standard GNN architecture as described in Equation~(\ref{eq:gnn_general}), and the ChebNet~\cite{defferrard2016convolutional} and ARMA~\cite{bianchi2019graph} models
which also aggregate information from the nodes' $k$-hop neighborhoods.
The ChebNet and ARMA models were implemented using the Pytorch Geometric library~\cite{fey2019fast}.

\noindent\textbf{Experimental setup.}
For each configuration, we generate $20$ graphs using the procedure described above.
For each graph, we perform $10$-fold cross validation.
We repeat the whole process $25$ times.
We measure the performance of the different algorithms using the following two evaluation metrics: ($1$) average accuracy and ($2$) average F$1$-score.

For the unsupervised algorithms, we learn an embedding for each node, and we predict the class label of each node in the test set using a $4$-nearest neighbors classifier.
For the proposed $k$-hop GNN model, the standard GNN model, ChebNet and ARMA, we train the models on the training set of each fold and use the models to classify the nodes of the test set.

For all the unsupervised algorithms, we use the default parameter values.
Specifically, for struc2vec, we set the probability that the random walks stays in current layer to 0.3, the dimensionality of the embeddings to $128$, the number of epochs to $5$, the number of walks per node to $10$, the walk length to $80$ and the context window size to $10$.
Furthermore, we make use of all approximations OPT1, OPT2, and OPT3.
For GraphWave, we use the multiscale version, set $d=50$ and use evenly spaced sampling points $t_i$ in range $[0,100]$.
Finally, for RolX, we did not use any approach for automatically detecting the number of different roles, but we directly provided the algorithm with the correct number of roles.
We use $2$ and $3$ neighborhood aggregation layers for the standard GNN, and $1$ layer for ChebNet, ARMA and the proposed $2$-hop and $3$-hop GNNs.
For ChebNet, we use polynomials of order $2$ and $3$, and for ARMA, we set the number of stacks $K$ to $2$ and the depth $T$ to $2$ and $3$.
The neighborhood aggregation layers of the $k$-hop GNN models and of the standard GNN models consist of MLPs with $2$ layers.
Batch normalization is applied to the output of every neighborhood aggregation layer.
The hidden-dimension size of the MLPs is chosen from $\{8,16,32\}$ and the dropout rate from $\{0.0, 0.2\}$.
To train all neural networks, we use the Adam optimizer with initial learning rate 0.01 and decay the learning rate by 0.5 every $50$ epochs.
We set the number of epochs to $200$.
We store the model that achieved the best validation accuracy into disk.
At the end of training, the model is retrieved from the disk, and we use it to classify the test instances.

\begin{table}[t]
    \centering
    \def\arraystretch{1.1}
    \begin{sc}
    \resizebox{\textwidth}{!}{
    \begin{tabular}{|l|c|c|c|c|c|c|c|c|} \hline
    \multirow{3}{*}{Method} & \multicolumn{8}{c|}{Configuration} \\ \cline{2-9}
    & \multicolumn{2}{c|}{basic} & \multicolumn{2}{c|}{basic perturbed} & \multicolumn{2}{c|}{varied} & \multicolumn{2}{c|}{varied perturbed} \\ \cline{2-9}
    & Accuracy & F$1$-score & Accuracy & F$1$-score & Accuracy & F$1$-score & Accuracy & F$1$-score \\ \hline
    RolX & \textbf{1.000} & \textbf{1.000} & 0.928 & 0.886 & 0.998 & 0.996 & 0.856 & 0.768 \\
    struc2vec & 0.784 & 0.708 & 0.703 & 0.632 & 0.738 & 0.592 & 0.573 & 0.412 \\
    GraphWave & 0.995 & 0.993 & 0.906 & 0.861 & 0.982 & 0.965 & 0.793 & 0.682 \\
    2-GNN & 0.997 & 0.994 & 0.920 & 0.876 & 0.990 & 0.979 & 0.852 & 0.753 \\
    3-GNN & 0.997 & 0.994 & 0.911 & 0.859 & 0.993 & 0.985 & 0.866 & 0.775 \\ 
    ChebNet (K=2) & 0.988 & 0.979 & 0.866 & 0.787 & 0.852 & 0.732 & 0.624 & 0.471  \\
    ChebNet (K=3) & 0.992 & 0.987 & 0.904 & 0.850 & 0.958 & 0.917 & 0.758 & 0.612 \\ 
    ARMA (T=2) & 0.996 & 0.992 & 0.914 & 0.861 & 0.982 & 0.961 & 0.839 & 0.728 \\ 
    ARMA (T=3) & 0.997 & 0.996 & 0.919 & 0.872 & 0.993 & 0.987 & 0.850 & 0.747 \\ \hline
    2-hop GNN & \textbf{1.000} & \textbf{1.000} & 0.961 & \textbf{0.934} & \textbf{0.999} & \textbf{0.999} & 0.948 & 0.910 \\
    3-hop GNN & \textbf{1.000} & \textbf{1.000} & \textbf{0.962} & \textbf{0.934} & 0.996 & 0.993 & \textbf{0.952} & \textbf{0.916} \\ \hline
    \end{tabular}
    }
    \end{sc}
    \caption{Performance of the baselines and the proposed $k$-hop GNN models for learning structural embeddings averaged over $20$ synthetically generated graphs for each configuration.}
    \label{tab:shapes_results}
\end{table}

\noindent\textbf{Results.}
Table~\ref{tab:shapes_results} shows that the instances of the proposed $k$-hop GNN architecture outperform all the baselines in the node classification task.
Interestingly, struc2vec is the worst-performing method in all configurations.
Specifically, in the least challenging configuration (``basic''), the $k$-hop GNN models, and RolX perform the best.
All these methods yield perfect performance, while the standard GNN models, the ChebNet and ARMA models, and GraphWave exhibit slightly worse performance.
In the presence of noise (``basic perturbed'' configuration), the performance of all methods degrades a lot.
The $3$-hop GNN model is the best-performing method followed by the $2$-hop GNN model, RolX, and the $2$-GNN model, in that order.
The remaining models achieve slightly lower accuracies and F$1$-scores than the above $4$ methods.
In the ``varied'' configuration, the $k$-hop GNN models are once again the best-performing methods along with RolX.
In the ``varied perturbed'' configuration, the $3$-hop GNN model yields the best performance.
The $2$-hop GNN model achieves slightly worse performance, while the performance of the remaining methods is much lower.
From the two noisy configurations, it is clear that the baseline methods are more prone to noise compared to the $k$-hop GNN model.
Furthermore, with regards to the two instances of the $k$-hop GNN model, the $3$-hop GNN model outperforms the $2$-hop GNN model in two configurations (\ie ``basic perturbed'' and ``varied perturbed''), is outperformed by the $2$-hop GNN model in one configuration (\ie ``varied''), while the two models achieve the same performance in the ``basic'' configuration.
  Finally, the proposed $k$-hop GNN models outperform the two standard GNN models in all experiments, thereby validating our theoretical results.
Overall, the proposed $k$-hop GNN model is robust and achieves good performance, demonstrating that it can learn high-quality node representations.

\subsubsection{Real-World Dataset}

\noindent\textbf{Datasets.}
We study the Enron dataset, an e-mail network encoding communication between employees in a company.
We expect structural equivalences in job titles due to corporate organizational hierarchy.
Nodes of the network represent Enron employees and edges correspond to e-mail communication between the employees.
There are $143$ nodes and 2583 edges in the emerging network.
An employee has one of seven functions in the company (\eg CEO, manager, etc.).
These functions provide ground-truth information about roles of the corresponding nodes in the network.

\noindent\textbf{Baselines.}
We compare the proposed model against the $9$ baseline algorithms which were presented above: ($1$) RolX, ($2$) struc2vec, ($3$) GraphWave, ($4$) $2$-GNN, ($5$) $3$-GNN, ($6$) ChebNet ($K=2$), ($7$) ChebNet ($K=3$), ($8$) ARMA ($T=2$), and ($9$) ARMA ($T=3$).

\begin{table}[t]
    \centering
    \scriptsize
    \def\arraystretch{1.1}
    \begin{sc}
    \begin{tabular}{|l|c|c|} \hline
    Method & Accuracy & F$1$-score \\ \hline
    RolX & 0.264 & 0.154 \\
    struc2vec & 0.323 & 0.190 \\
    GraphWave & 0.257 & 0.149 \\
    2-GNN & 0.357 & 0.183 \\
    3-GNN & 0.366 & 0.195 \\ 
    ChebNet (K=2) & 0.342 & 0.179 \\
    ChebNet (K=3) & 0.360 & 0.191 \\ 
    ARMA (T=2) & 0.374 & 0.192 \\ 
    ARMA (T=3) & \textbf{0.376} & 0.190 \\ \hline
    2-hop GNN & 0.366 & \textbf{0.198} \\
    3-hop GNN & 0.327 & 0.171 \\ \hline
    \end{tabular}
    \end{sc}
    \caption{Performance of the baselines and the proposed $k$-hop GNN models for learning structural embeddings on the Enron dataset.}
    \label{tab:enron}
\end{table}

\noindent\textbf{Experimental setup.}
We perform $10$-fold cross validation, and repeat the whole process $20$ times.
For each algorithm, we report ($1$) its average accuracy, and ($2$) its average F$1$-score.

For the unsupervised algorithms, we first learn an embedding for each node, and then, for each fold, we use a $4$-nearest neighbors classifier to predict the job titles of the nodes of the test set.
For the proposed $k$-hop GNN model, the standard GNN model, ChebNet and ARMA, we train the models on the training set of each fold and use the models to classify the nodes of the test set.

For all algorithms, we set/optimize their hyperparameters as described in subsection~\ref{sec:experiments_node_synthetic}.

\noindent\textbf{Results.}
We can see in Table~\ref{tab:enron} that the supervised neural network models outperformed the unsupervised algorithms on this dataset.
From the unsupervised algorithms, only struc2vec achieved performance comparable to that of the supervised models.
In terms of accuracy, ARMA is the best-performing method.
Both ARMA ($T=2$) and ARMA ($T=3$) outperform all the other methods on the Enron dataset.
On the other hand, in terms of F$1$-score, $2$-hop GNN achieves the best performance among the different methods, followed by $3$-GNN and ARMA ($T=2$).
Surprisingly, the $3$-hop GNN model performs much worse than the $2$-hop GNN model.
We hypothesize that this is related to the structure of the e-mail network.

\subsection{Graph Classification}

We next apply the proposed model to the problem of graph classification., \ie the supervised learning task of assigning a graph to a set of predefined categories.
Specifically, given a set of graphs $\{G_1,\ldots,G_N \} \subseteq \mathcal{G}$ and their class labels $\{y_1,\ldots,y_N\}$, the goal is to learn a representation vector $h_{G_i}$ such that the class label of every graph of the test set can be predicted as $y_{i} = f(h_{G_i})$.
For this task, we are going to evaluate the performance of the proposed model in two different types of datasets: ($1$) synthetic datasets containing graphs that satisfy or do not satisfy the considered graph properties, and ($2$) standard widely-used datasets from real-world scenarios.

\subsubsection{Synthetic Datasets}

\noindent\textbf{Datasets.}
In order to investigate if the proposed model can distinguish triangle-freeness, bipartiteness and connectivity, we created three synthetic datasets.
Each one consists of $800~4$-regular graphs of 60 nodes each and is assigned a class label which denotes whether it satisfies the corresponding property or not (\ie binary classification task).
All the nodes are assigned identical labels.
Furthermore, all three datasets are balanced, \ie half of the graphs ($400$ graphs) satisfy the examined graph property, while the rest of the graphs ($400$ graphs) do not satisfy it.

\noindent\textbf{Baselines.}
We compare our model against the standard GNN architecture of Equations~(\ref{eq:gnn_general}) and (\ref{eq:readout}), and against the ChebNet~\cite{defferrard2016convolutional} and ARMA~\cite{bianchi2019graph} models.
Based on our theoretical results, we expect the standard GNN model to perform worse than the proposed model on these $3$ datasets.
On the other hand, since the ChebNet and ARMA models aggregate information based on the $k$-hop neighborhood of each node, these models may be able to capture these properties.

\noindent\textbf{Experimental setup.}
For ChebNet, ARMA and the proposed $k$-hop GNN models ($2$-hop GNN and $3$-hop GNN), we used a single neighborhood aggregation layer, while for the standard GNN, we used $2$ and $3$ layers.
The hidden-dimension size of these layers was chosen from $\{16, 32, 64\}$. 
For ChebNet, we used polynomials of order $2$ and $3$, and for ARMA, we set the number of stacks $K$ to $2$ and the depth $T$ to $2$ and $3$.
To generate graph representations, we employed a readout function that sums the vector representations of the nodes.
The generated graph representations are then fed into a two layer MLP, with a softmax output.
We used the \texttt{ReLU} activation function, and set the batch size to $32$.
We used the Adam optimizer with a learning rate of $10^{-3}$, while we set the number of epochs to $100$.

\begin{table}[t]
    \centering
    \scriptsize
    \def\arraystretch{1.1}
    \begin{sc}
    \begin{tabular}{|l|ccc|} \hline
    \multirow{3}{*}{Method} & \multicolumn{3}{c|}{Property} \\ \cline{2-4}
    & \multirow{2}{*}{Connectivity} & \multirow{2}{*}{Bipartiteness} & Triangle \\
    & & & freeness \\ \hline \hline
    2-GNN & 55.00 $\pm$ 5.30 & 53.78 $\pm$ 2.61 & 51.87 $\pm$ 7.43 \\
    3-GNN & 56.20 $\pm$ 2.28 & 58.13 $\pm$ 2.10 & 55.90 $\pm$ 4.44 \\
    ChebNet (K=2)& 56.37 $\pm$ 7.76 & 50.33 $\pm$ 1.20 & 53.12 $\pm$ 6.35 \\
    Chebnet (K=3) & 57.62 $\pm$ 3.84 & 51.98 $\pm$ 3.56 & 54.75 $\pm$ 7.14 \\
    ARMA (T=2) & 55.55 $\pm$ 5.59  & 54.50 $\pm$ 4.61 & 53.00 $\pm$ 3.18 \\
    ARMA (T=3) & 55.63 $\pm$ 5.69 & 53.92 $\pm$ 3.22  & 54.25 $\pm$ 5.80 \\
    \hline \hline
    2-hop GNN & 81.24 $\pm$ 5.22 & 84.69 $\pm$ 1.74 & \textbf{84.06} $\pm$ 2.12 \\ 
    3-hop GNN & \textbf{94.77} $\pm$ 3.41 & \textbf{91.12} $\pm$ 2.76 & 82.53 $\pm$ 5.33 \\ \hline
    \end{tabular}
    \end{sc}
    \caption{Average classification accuracy of the proposed $k$-hop GNN models and the baselines on the $3$ synthetic datasets.}
    \label{tab:classification_synthetic}
\end{table}

\noindent\textbf{Results.}
We report in Table~\ref{tab:classification_synthetic} average prediction accuracies across the $10$ folds.
It is clear that the standard GNN architectures are unable to distinguish the $3$ graph properties.
Specifically, they all achieve an average accuracy slightly greater than $50\%$ on all three datasets. The ChebNet and ARMA models, even though they aggregate information from each node's $k$-hop neighborhood, they are also unable to distinguish the $3$ properties.
On the other hand, the proposed $k$-hop GNN architectures achieved much higher average accuracies, indicating that the proposed architecture can distinguish the 3 properties in regular graphs.
In the case of bipartiteness and connectivity, the $3$-hop GNN model achieved very high accuracy.
The performance of the $2$-hop GNN model was slightly worse than that of the $3$-hop GNN model.
However, the former managed to better distinguish triangle-free graphs than the latter.
Overall, we can conclude that in contrast to the GNN, ChebNet and ARMA architectures, the proposed model leads to more expressive node representations.

\subsubsection{Real-World Datasets}

We also evaluate the proposed $k$-hop GNN model on standard graph classification datasets derived from bio/chemoinformatics, and from social networks.

\noindent\textbf{Datasets.}
We use the following $3$ datasets from bioinformatics and chemoinformatics: ($1$) MUTAG, ($2$) PROTEINS, ($3$) NCI1.
We also use the following $2$ social interaction datasets: ($1$) IMDB-BINARY, ($2$) IMDB-MULTI.
We next give more details about the employed datasets\footnote{The datasets, further references and statistics are available at \url{https://ls11-www.cs.tu-dortmund.de/staff/morris/graphkerneldatasets}.}.

MUTAG consists of $188$ mutagenic aromatic and heteroaromatic nitro compounds.
The task is to predict whether or not each chemical compound has mutagenic effect on the Gram-negative bacterium Salmonella typhimurium.
PROTEINS comprises of 1113 proteins represented as graphs where vertices are secondary structure elements and there is an edge between two vertices if they are neighbors in the amino-acid sequence or in 3D space.
The task is to classify proteins into enzymes and non-enzymes.
NCI1 contains 4110 chemical compounds screened for activity against non-small cell lung cancer and ovarian cancer cell lines.
IMDB-BINARY and IMDB-MULTI contain 1000 and 1500 movie collaboration graphs, respectively.
Each graph is the ego-network of an actor/actress, and the task is to predict which genre an ego-network belongs to.

Note that the social network graphs are unlabeled, while all other graph datasets come with discrete node labels.
Hence, for the social network graphs, we assign an one-dimensional feature to each node which is equal to its degree, while for the remaining datasets, we use one-hot encodings of the node labels.

Table~\ref{tab:dataset_statistics} provides a summary of the employed datasets.
\begin{table}[t]
    \centering
    \scriptsize
    \begin{sc}
    \def\arraystretch{1.2}
    \begin{tabular}{|l|c|c|d{3.2}|d{3.2}|d{5.2}|cc|} \hline
    \multirow{3}{*}{Dataset} & \multicolumn{7}{c|}{\multirow{1}{*}{Statistics}} \\ \cline{2-8} 
    & \multirow{2}{*}{\#Graphs} & \multirow{2}{*}{\#Classes} & \multicolumn{1}{c|}{Max Class} & \multicolumn{1}{c|}{Avg.} & \multicolumn{1}{c|}{Avg.} & \multicolumn{2}{c|}{Labels} \\
    & & & \multicolumn{1}{c|}{Imbalance} & \multicolumn{1}{c|}{\#Nodes} & \multicolumn{1}{c|}{\#Edges} & \multicolumn{2}{c|}{(Num.)} \\ \hline
    MUTAG                         & 188             & 2      & 1:1.98  & 17.93                & 19.79                & + & (7)  \\ \hline
    PROTEINS                      & 1,113           & 2      & 1:1.47  & 39.06                & 72.82                & + &(3)   \\ \hline
    NCI1                          & 4,110           & 2      & 1:1     & 29.87                & 32.30                & + & (37) \\ \hline
    IMDB-BINARY                   & 1,000           & 2      & 1:1     & 19.77                & 96.53                & --&      \\ \hline
    IMDB-MULTI                    & 1,500           & 3      & 1:1     & 13.00                & 65.94                & --&      \\ \hline
    \end{tabular}
    \end{sc}
    \caption{Summary of the $5$ datasets used in our experiments. The ``Max Class Imbalance'' column indicates the ratio of the size of the smallest class of the dataset to the size of its largest class.}
    \label{tab:dataset_statistics}
\end{table}

\noindent\textbf{Baselines.}
We compare our methods against four graph kernels: ($1$) the graphlet kernel (GK)~\cite{shervashidze2009efficient}, ($2$) the shortest-path kernel (SP)~\cite{borgwardt2005shortest}, ($3$) the
Weisfeiler--Lehman subtree kernel (WL)~\cite{shervashidze2011weisfeiler}, and ($4$) the Weisfeiler--Lehman Optimal Assignment kernel (WL-OA)~\cite{kriege2016valid}.
The first three kernels are available in the GraKeL library~\cite{siglidis2020grakel}, while for WL-OA we used the code provided by the authors.
Besides graph kernels, we also compare the proposed model against the basic GNN architecture of Equations~(\ref{eq:gnn_general}) and (\ref{eq:readout}), against GS-SVM~\cite{gao2019geometric} which makes use of geometric scattering features, and against the following state-of-the-art deep learning architectures: ($1$) ChebNet~\cite{defferrard2016convolutional}, ($2$) ARMA~\cite{bianchi2019graph}, ($3$)
PatchySan~\cite{niepert2016learning}, ($4$) Deep Graph CNN (DGCNN)~\cite{zhang2018end}, ($5$) CapsGNN~\cite{xinyi2018capsule}, and ($6$)
$1$-$2$-$3$-GNN~\cite{morris2018weisfeiler}.
For ChebNet and ARMA, we use the implementations contained in the PyTorch Geometric library~\cite{fey2019fast}.
For GS-SVM and the rest of the deep learning methods, we compare against the accuracies reported in the original papers.

\noindent\textbf{Experimental Setup.}
We performed $10$-fold cross-validation where $10\%$ of the graphs of each training fold was used as a validation set.
The whole process was repeated $10$ times for each dataset and each approach.

We chose parameters for the graph kernels as follows.
For the Weisfeiler--Lehman subtree kernel and for the Weisfeiler--Lehman optimal assignment kernel, we chose the number of iterations from $h=\{4,5,6,7\}$, while the graphlet kernel that we implemented samples $500$ graphlets of size up to $6$ from each graph.
For ChebNet, ARMA and the proposed $k$-hop GNN models ($2$-hop GNN and $3$-hop GNN), we used a single neighborhood aggregation layer, while for the standard GNN, we used $2$ and $3$ layers.
The parameters of the neighborhood aggregation layers correspond to MLPs with $2$ layers.
Batch normalization is applied to the output of every neighborhood aggregation layer.
The hidden-dimension size of the MLPs was chosen from $\{16, 32, 64\}$.
For ChebNet, we used polynomials of order $2$ and $3$, and for ARMA, we set the number of stacks $K$ to $2$ and the depth $T$ to $2$ and $3$.
To generate graph representations, we employed a readout function that sums the vector representations of the nodes.
The generated graph representations are then fed into a two layer MLP, with a softmax output.
We used the \texttt{ReLU} activation function, and we chose the batch size from $\{32, 64, 128\}$.
We used the Adam optimizer with an initial learning rate of $10^{-3}$ and decay the learning rate by $0.5$ every $50$ epochs.
We set the number of epochs to $500$, and we select the epoch with the best validation accuracy.
The source code of the proposed model is available at \url{https://github.com/giannisnik/k-hop-gnns}.

\begin{table}[t]
    \centering
    \def\arraystretch{1.1}
    \begin{sc}
    \resizebox{\textwidth}{!}{
    \begin{tabular}{|l|cccccc|} \hline
    \multirow{2}{*}{\backslashbox{Method}{Dataset}} & \multirow{2}{*}{MUTAG} & \multirow{2}{*}{PROTEINS} & \multirow{2}{*}{NCI$1$} & IMDB & IMDB & Average \\
    & & & & BINARY & MULTI & Rank \\ \hline \hline
    GK & 69.97 ($\pm$ 2.22) & 71.23 ($\pm$ 0.38) & 65.47 ($\pm$ 0.14) & 60.33 ($\pm$ 0.25) & 36.53 ($\pm$ 0.93) & 16.8 \\
    SP & 84.03 ($\pm$ 1.49) & 75.36 ($\pm$ 0.61) & 72.85 ($\pm$ 0.24) & 60.21 ($\pm$ 0.58) & 39.62 ($\pm$ 0.57) & 12.8 \\
    WL & 83.63 ($\pm$ 1.57) & 73.12 ($\pm$ 0.52) & 84.42 ($\pm$ 0.25) & 73.36 ($\pm$ 0.38) & \textbf{51.06} ($\pm$ 0.47) & 6.8 \\
    WL-OA & 86.63 ($\pm$ 1.49) & 75.35 ($\pm$ 0.45) & \textbf{85.74} ($\pm$ 0.37) & 73.61 ($\pm$ 0.60) & 50.48 ($\pm$ 0.33) & 3.0 \\ \hline \hline
    GS-SVM & 83.57 ($\pm$ 6.75) & 74.11 ($\pm$ 4.02) & 79.14 ($\pm$ 1.28) & 71.20 ($\pm$ 3.25) & 48.73 ($\pm$ 2.32) & 10.8 \\ \hline \hline
    2-GNN & 85.92 ($\pm$ 2.19) & 75.24 ($\pm$ 0.45) & 76.32 ($\pm$ 0.41) & 71.40 ($\pm$ 0.74) & 47.73 ($\pm$ 0.86) & 8.8 \\
    3-GNN & 85.74 ($\pm$ 1.48) & 74.59 ($\pm$ 0.71) & 79.62 ($\pm$ 0.45) & 71.60 ($\pm$ 0.84) & 47.33 ($\pm$ 1.01) & 9.4 \\
    ChebyNet (K=2) & 85.33 ($\pm$ 1.42) & 74.72 ($\pm$ 0.97) & 78.97 ($\pm$ 0.35) & 71.08 ($\pm$ 0.51) &  47.08 ($\pm$ 0.60) & 10.6 \\
    ChebyNet (K=3) & 82.49 ($\pm$ 1.52) & 74.81 ($\pm$ 0.82) & 81.01 ($\pm$ 0.39) & 70.90 ($\pm$ 0.73) & 46.66 ($\pm$ 0.59) & 10.8 \\
    ARMA (T=2) & 82.98 ($\pm$ 1.90) & 74.84 ($\pm$ 0.59) & 80.83 ($\pm$ 0.42) & 70.62 ($\pm$ 0.95) & 46.10 ($\pm$ 0.82) & 11.2 \\
    ARMA (T=3) & 81.52 ($\pm$ 1.22) & 74.74 ($\pm$ 0.67) & 81.34 ($\pm$ 0.38) & 70.52 ($\pm$ 0.71) & 46.12 ($\pm$ 0.98) & 11.6 \\
    PatchySan ($k=10$) & \textbf{88.95} ($\pm$ 4.37) & 75.00 ($\pm$ 2.51) & 76.34 ($\pm$ 1.68) & 71.00 ($\pm$ 2.29) & 45.23 ($\pm$ 2.84) & 9.6 \\
    DGCNN & 85.83 ($\pm$ 1.66) & 75.54 ($\pm$ 0.94) & 74.44 ($\pm$ 0.47) & 70.03 ($\pm$ 0.86) & 47.83 ($\pm$ 0.85) & 9.6 \\ 
    CapsGNN & 86.67 ($\pm$ 6.88) & \textbf{76.28} ($\pm$ 3.63) & 78.35 ($\pm$ 1.55) & 73.10 ($\pm$ 4.83) & 50.27 ($\pm$ 2.65) & 5.0 \\
    1-2-3-GNN & 86.1 & 75.5 & 76.2 & \textbf{74.2} & 49.5 & 6.0 \\ \hline \hline
    2-hop GNN & 87.93 ($\pm$ 1.22) & 75.03 ($\pm$ 0.42) & 79.31 ($\pm$ 0.57) & 73.33 ($\pm$ 0.30) & 49.79 ($\pm$ 0.25) & 5.4 \\
    3-hop GNN & 87.56 ($\pm$ 0.72) & 75.28 ($\pm$ 0.36) & 80.61 ($\pm$ 0.34) &  - &  - & 4.8 \\ \hline
    \end{tabular}
    }
    \end{sc}
    \caption{Average classification accuracy ($\pm$ standard deviation) of the baselines and the proposed $k$-hop GNN models on the $5$ graph classification benchmark datasets.  The  ``Average Rank'' column illustrates the average rank of each method. The lower the average rank, the better the overall performance of the method.}
    \label{tab:classification}
\end{table}

\noindent\textbf{Results.}
We report in Table~\ref{tab:classification} average prediction accuracies and standard deviations across the $10$ repetitions.
Note that the graphs contained in the IMDB-BINARY and IMDB-MULTI datasets correspond to the ego-networks of actors/actresses.
The diameter of these graphs is at most equal to $2$, and therefore, the $3$-hop neighborhoods of the nodes are identical to their $2$-hop neighborhoods.
This is why we do not report the performance of the $3$-hop GNN on these datasets.

In general, we observe that the variants of the proposed model achieve high levels of performance.
Specifically, they achieve the second best performance on MUTAG, the fourth best performance on IMDB-BINARY and IMDB-MULTI, and the sixth best performance on the NCI1 and PROTEINS datasets.
On most datasets, the proposed model yields only slightly worse accuracies compared to the best performing method, the WL-OA kernel.
Interestingly, the two $k$-hop GNN models perform equally well in general.
More specifically, the $3$-hop GNN achieves slightly better accuracy than the $2$-hop GNN on most datasets.
However, the difference in performance is not very large.
Furthermore, it should be mentioned that the proposed $k$-hop GNN models outperform the two standard GNN models on all datasets, demonstrating their superiority.
Overall, the proposed architecture yields good performance, demonstrating that it can learn not only high-quality node representations, but also graph representations.

\noindent\textbf{Runtime Analysis.}
We also compare the running time of the proposed model against that of the standard GNN model on the five real-world datasets.
We report in Table~\ref{tab:runtime_epoch} the average time per epoch, and in Table~\ref{tab:runtime_preprocessing} the preprocessing time (both in seconds).
The hyperparameters of all models were set to the same values.

\begin{table}[t]
    \centering
    \scriptsize
    \def\arraystretch{1.1}
    \begin{sc}
    \begin{tabular}{|l|ccccc|} \hline
    \multirow{2}{*}{\backslashbox{Method}{Dataset}} & \multirow{2}{*}{MUTAG} & \multirow{2}{*}{PROTEINS} & \multirow{2}{*}{NCI$1$} & IMDB & IMDB \\
    & & & & BINARY & MULTI \\ \hline \hline
    2-GNN & 0.01 & 0.08 & 0.29 & 0.07 & 0.10 \\
    3-GNN & 0.02 & 0.12 & 0.38 & 0.09 & 0.13 \\ \hline \hline
    2-hop GNN & 0.03 & 0.18 & 0.59 & 0.18 & 0.22 \\
    3-hop GNN & 0.04 & 0.27 & 0.82 & 0.18 & 0.22 \\ \hline
    \end{tabular}
    \end{sc}
    \caption{Average running time per epoch (in seconds) of the proposed $k$-hop GNN models and the standard GNN models on the $5$ graph classification benchmark datasets.}
    \label{tab:runtime_epoch}
\end{table}

\begin{table}[t]
    \centering
    \scriptsize
    \def\arraystretch{1.1}
    \begin{sc}
    \begin{tabular}{|l|ccccc|} \hline
    \multirow{2}{*}{\backslashbox{Method}{Dataset}} & \multirow{2}{*}{MUTAG} & \multirow{2}{*}{PROTEINS} & \multirow{2}{*}{NCI$1$} & IMDB & IMDB \\
    & & & & BINARY & MULTI \\ \hline \hline
    2-GNN & 2.72 & 3.36 & 4.78 & 3.20 & 3.35 \\
    3-GNN & 2.72 & 3.38 & 4.79 & 3.21 & 3.33 \\ \hline \hline
    2-hop GNN & 3.11 & 13.02 & 19.63 & 31.38 & 28.36 \\
    3-hop GNN & 3.28 & 19.00 & 27.50 & 31.37 & 28.37 \\ \hline
    \end{tabular}
    \end{sc}
    \caption{Preprocessing time (in seconds) of the proposed $k$-hop GNN models and the standard GNN models on the $5$ graph classification benchmark datasets.}
    \label{tab:runtime_preprocessing}
\end{table}

As expected, the standard GNN is faster than the proposed model both in terms of time per epoch and in terms of the preprocessing time.
However, the running time of the proposed model is by no means prohibitive.
In general, the average time per epoch of the proposed models is twice as high as that of the corresponding standard GNN models.
Furthermore, we should note that the preprocessing time of the standard GNN is independent of the value of $k$, while for the proposed model, it increases as $k$ increases (since larger neighborhoods need to be processed).

\section{Conclusion}\label{sec:conclusion}
In this paper, we analyzed the expressive power of GNNs, showing that a wide class of GNN architectures cannot identify fundamental properties of graphs.
We also proposed the $k$-hop GNN model which aggregates information from the nodes' $k$-hop neighborhoods, and is capable of identifying graph properties that are not captured by standard GNNs.
We evaluated the proposed model on node and graph classification datasets, where it achieved results competitive with state-of-the-art algorithms.

\section*{Acknowledgment}
Giannis Nikolentzos was supported by the project ``ESIGMA'' (ANR-17-CE40-0028).
The authors gratefully acknowledge the support of the NVIDIA Corporation, who donated the Titan Xp GPU that was used for this research in the context of the NVIDIA GPU Grant program.

\section*{Appendix}
\subsection*{Proof of Theorem~\ref{thm:expressiveness}}

We assume that the feature vectors of the nodes come from a countable set.
This set may correspond to a subset of an uncountable set such as $\mathbb{R}^d$.
Furthermore, the feature vectors of a set of nodes form a multiset (\ie since some nodes may have identical feature vectors).

We will show that if a graph has one of the three considered properties (\ie triangle-freeness, bipartiteness, and connectivity), some of its nodes can be mapped to different feature vectors compared to the nodes of a graph that does not have the property.
Then, by applying an injective readout function, the two graphs can also be mapped to different feature vectors.

For simplicity of presentation, we will assume that the proposed model consists of a single neighborhood aggregation layer.
The same results also hold for multiple neighborhood aggregation layers.
We first show that the aggregation scheme that our model employs can represent universal functions over the pairs of a node and the multiset
of its neighbors.
The following Lemma generalizes the setting in~\cite{xu2018powerful}.
\begin{lemma}
    Assume $\mathcal{X}$ is countable, and let $r \in \mathbb{N}$.
    There exist functions $f : \mathcal{X} \rightarrow \mathbb{R}^d$ and $f' : \mathcal{X} \rightarrow \mathbb{R}^d$, such that $h_i(c, X) = f(c)+\sum_{x \in X} f'(x)$ is unique for each $i \in \{0,\ldots,r\}$ and each pair $(c, X)$, where $c \in \mathcal{X}$ and $X \subset \mathcal{X}$ is a finite multiset.
    Moreover, any function $g_i$ over such pairs can be decomposed as $g_i(c,X) = \phi\bigl(f(c)+\sum_{x \in X} f'(x)\bigr)$ for some function $\phi$.
    \label{eq:lemma2}
\end{lemma}
\begin{proof}
    We first show that there exists a mapping $f'$ such that $\sum_{x \in X} f'(x)$ is unique for each finite multiset $X$.
    Because $\mathcal{X}$ is countable, there exists a mapping $Z : \mathcal{X} \rightarrow \mathbb{N}$ from $x \in \mathcal{X}$ to natural numbers.
    Because the multisets $X$ are finite, there exists a number $N \in \mathbb{N}$ such that $|X| < N$ for all $X$.
    Then, an example of such $f'$ is $f'(x) = N^{-Z(x)}$.
    The sum of the above function for all finite multisets $X$ takes values less than $1$, \ie $\sum_{x \in X} f'(x) < 1$.
    Hence, if we also set $f(x) = Z(x)+(r-i) |\mathcal{X}|$, then it holds that $h_i(c, X) = f(c)+\sum_{x \in X} f'(x)$ is an injective function over pairs of elements and multisets, and is also unique for each $i \in \{ 0,\ldots,r \}$.

    For any function $g_i$ over the pairs $(c,X)$, we can construct such $\phi$ for the desired decomposition by letting $g_i(c, X) = \phi\bigl(f(c)+\sum_{x \in X} f'(x)\bigr)$.
    Note that such $\phi$ is well-defined because $h_i(c, X) = f(c)+\sum_{x \in X} f'(x)$ is injective.
\end{proof}
In our setting, the $\text{UPDATE}$ modules correspond to $g_i(c, X)$ functions.
These modules use multi-layer perceptrons (MLPs) to model and learn $f,f'$ and $\phi$ in the above Lemma, thanks to the universal approximation theorem~\cite{hornik1989multilayer}.
Note that given two nodes $v,v'$, if a node $u \in \mathcal{N}_k(v)$ obtains a representation that is never obtained by any node $u' \in \mathcal{N}_k(v')$, then based on the above Lemma, there exist $\text{UPDATE}$ modules such that the root nodes $v,v'$ are assigned different representations.
Hence, for all three properties, it is sufficient to show that at some point of the algorithm, a node of the graph that satisfies the property can obtain a representation that is never obtained by any node of a graph that does not satisfy the property.
\paragraph*{Triangle-freeness}

If a graph is not triangle-free, then there exist at least three nodes whose $1$-hop neighborhoods contain a triangle.
Let $v$ be such a node.
Then, clearly there are at least two nodes $u \in R_1(v)$ which are connected to each other by an edge.
The representations of these nodes are updated as follows: $x_u = \text{UPDATE}_{1, within}^{(0)} \bigl( u, \mathcal{D} \bigr)$.
On the other hand, no such update takes place in the case of triangle-free graphs since $\mathcal{D} = \varnothing$.
Hence, based on Lemma~\ref{eq:lemma2}, the above $\text{UPDATE}_{1, within}^{(0)}$ module can generate different representations for the nodes that participate in a triangle from the representations of the nodes of the neighborhood subgraph of each node of a triangle-free graph.
\paragraph*{Connectivity}
    
Let $G$ be a disconnected graph and $C$ its component which has the minimum diameter $\delta_{min}$.
Then, for an arbitrary node $v$ of component $C$, it holds that $R_i(v) = \varnothing$ for all $i > \delta_{min}$.
On the other hand, if the graph is connected, for some node $v$, it holds that $|R_i(v)|  > 0$ for all $i \leq \delta$ where $\delta > \delta_{min}$ is the diameter of the connected graph.
Hence, the representations of the nodes $u \in R_i(v)$ and $u' \in R_{i-1}(v)$ are updated as $x_u = \text{UPDATE}_{i, within}^{(0)} \bigl( u, \mathcal{D} \bigr)$ and $x_{u'} = \text{UPDATE}_{i-1, across}^{(0)} \bigl( u', \mathcal{D}' \bigr)$, respectively.
Based on Lemma~\ref{eq:lemma2}, the above two $\text{UPDATE}$ modules can generate different representations for the nodes of a neighborhood subgraph of a disconnected graph compared to those of the nodes of the neighborhood subgraph of a connected graph.
\paragraph*{Bipartiteness}

It is well-known that a graph is bipartite if and only if it does not contain an odd cycle.
If $G$ is bipartite and $l$ is the length of the smallest odd cycle in $G$, then the $k$-hop neighborhood subgraphs ($k \geq \frac{l-1}{2}$) of more than one nodes contain a cycle of odd length.
According to Lemma~\ref{lemma:bipartite}, the $k$-hop neighborhood subgraph of a node $v$ contains a cycle of odd length if and only if the shortest path lengths from two adjacent nodes $u,w \in \mathcal{N}_k(v)$ to $v$ are identical.
In other words, there exist two nodes both at the same level $i$ of the $k$-hop neighborhood subgraph of node $v$ that are connected to each other with an edge.
During the process of updating the representation of the root $v$, the feature vectors of these nodes are also updated as follows: $x_u = \text{UPDATE}_{i, within}^{(0)} \bigl( u, \mathcal{D} \bigr)$.
This update does not take place in the case of a non-bipartite graph since $\mathcal{D} = \varnothing$ for all nodes of all neighborhood subgraphs.
Based on Lemma~\ref{eq:lemma2}, these nodes can obtain different representations from all the representations of the nodes of a neighborhood subgraph extracted from a bipartite graph.

\begin{lemma}
    Let $G_v^k$ be the $k$-hop neighborhood subgraph of a node $v$.
    Then, $G_v^k$ contains a cycle of odd length if and only if the shortest path lengths from two adjacent nodes $u,w \in \mathcal{N}_k(v)$ to $v$ are identical.
    \label{lemma:bipartite}
\end{lemma}

\begin{proof}
    Let $u,w$ be two vertices such that $u,w \in \mathcal{N}_k(v)$.
    Assume that the shortest path lengths between each of these two vertices and the root $v$ are identical and equal to $d \in \mathbb{N}$ such that $d\leq k$.
    If $u$ and $w$ are connected by an edge, then $G_v^k$ contains a cycle of length $2d + 1$ which is clearly an odd number.
    This proves the first statement.
    For the second statement, assume that $G_v^k$ contains a cycle of odd length and there is no edge between two vertices whose shortest path lengths from the root $v$ are identical.
    Then, from all the nodes of the cycle, there is a single node $u$ such that the shortest path distance from the root $v$ to $u$ is maximum, and another node $w$ such that the shortest path distance from the root $v$ to $w$ is minimum ($w$ could correspond to the root itself).
    Since this is a cycle, there are two paths from $w$ to $u$ of length $d$.
    Hence, the length of the cycle is $2d$ which is an even number, contradicting the assumption.
\end{proof}

\bibliographystyle{plain}
\bibliography{biblio}

\end{document}